\documentclass[runningheads]{llncs}

\usepackage[T1]{fontenc}
\usepackage{graphicx}
\usepackage{amsmath,amssymb}
\usepackage{booktabs}
\usepackage{url}
\usepackage{tikz}
\usetikzlibrary{arrows.meta,positioning,fit,calc}

\begin{document}

\title{CADRE: Stable, Parameter-Efficient Adaptation of Medical Vision-Language
Models with Bounded Forgetting and Prior Drift}
\titlerunning{CADRE: Stable, Parameter-Efficient Adaptation of Medical VLMs}

\author{Rishabh Jha\inst{1} \and Amrita Singh\inst{2} \and Prashanna Chudal\inst{2}}
%index{Jha, Rishabh}
%index{Singh, Amrita}
%index{Chudal, Prashanna}
\authorrunning{R. Jha et al.}
\institute{Department of Computer Science, University of Victoria,
British Columbia, Canada\\ \email{rishabhjha12@uvic.ca}
\and Mindrisers Consortium, Kathmandu, Nepal}

\maketitle

\begin{abstract}
Adapting a medical vision-language model (VLM) such as BiomedCLIP to a clinical service
is as much a stability problem as an accuracy problem: an update for a new imaging
modality can silently lose competence on modalities already handled and drift from the
trustworthy pretrained prior toward modality-specific shortcuts. Elastic weight
consolidation (EWC), the standard remedy, applies one fixed multiplier to an
unnormalized, accumulating Fisher, so the effective constraint depends on task loss
scale and on how many modalities have been consolidated, and the same hyper-parameter is
too weak under one arrival order and too strong under another. We show this is a
\emph{scale} effect and remove it by construction. CADRE keeps the backbone frozen and
adapts through low-rank adaptation (LoRA), regularized by an online, self-scaling
consolidation term over a sum-normalised Fisher plus an anchor-to-prior penalty on
embedding drift; consolidation mass is then bounded independently of the number of
modalities and the self-scaled penalty is invariant to rescaling. On a deliberately
dissimilar three-modality stress test (breast histopathology, ultrasound, chest
radiography) at $\approx$0.23\% of parameters, CADRE reduces forgetting roughly
sevenfold against the strongest LoRA-family baseline ($0.075\!\rightarrow\!0.011$;
paired $p{=}0.023$, $n{=}6$) and is the only method compared with positive backward
transfer.

\end{abstract}

\section{Introduction}
Foundation vision-language models (VLMs) pretrained on large biomedical image-text
corpora generalize across many tasks, yet a clinical deployment rarely matches the
pretraining distribution: a breast-imaging service may want one model to read
histopathology, ultrasound and radiographs, and to keep improving as new modalities
arrive. Adapting naively to each new modality in turn is unsafe in two specific senses.
\emph{Silent competence loss}: sequential adaptation degrades performance on earlier
modalities~\cite{mccloskey1989catastrophic}, which in a clinic is invisible, as an
ultrasound update can quietly weaken histopathology reading and surface only as missed
diagnoses. \emph{Silent prior drift and order fragility}: adapters can move a model
toward dataset-specific shortcuts, and elastic weight consolidation (EWC), the standard
regularization remedy, is itself fragile: a fixed global penalty on an unnormalized,
accumulating importance estimate over-constrains whichever modality arrives last, so the
same method can be well behaved under one order and not under another.

That second point is the technical core of this paper. Rather than asking how high
accuracy is after adaptation, we ask \emph{what adaptation preserves about the model we
already trusted}, requiring three monitorable properties: retained competence, bounded
drift from the pretrained prior, and usable calibration (S1 to S3,
Sec.~\ref{sec:method}). These align with clinical-safety desiderata but do not guarantee
safe deployment. \textbf{CADRE} (Clinician-anchored, Domain-Robust, Efficient
adaptation) freezes the BiomedCLIP backbone and adapts through low-rank adaptation
(LoRA)~\cite{hu2021lora}, regularized by an online self-scaling EWC term and an
anchor-to-prior penalty, with light label smoothing and evaluation-time weight averaging
(Fig.~\ref{fig:method}).

\smallskip\noindent\textbf{Contributions.} (1) A \emph{diagnosis}: two coupled
\emph{scale} problems in online EWC (loss-scale dependence of the raw Fisher and
unbounded accumulation of the running estimate) account for much of its sensitivity to
arrival order. (2) A \emph{repair} with two guarantees: sum-normalising each
consolidated Fisher bounds total consolidation mass independently of the number of
modalities (Proposition~\ref{prop:mass}), and setting the multiplier as a fixed fraction
of the task loss makes the trade-off invariant to penalty rescaling
(Proposition~\ref{prop:scale}); together these remove a hand-tuned hyper-parameter
rather than replacing it. (3) A stability-oriented protocol with leakage controls,
multi-seed and multi-order runs, paired testing and component attribution; against the
LoRA-family baselines evaluated, CADRE attains the highest accuracy,
Stability-Plasticity Quotient and backward transfer and the lowest forgetting,
significantly so against plain LoRA and LoRA$+$Anchor, with the LoRA$+$EWC comparison
unresolved.

\section{Related Work}
\textbf{Parameter-efficient fine-tuning (PEFT) and continual learning.}
Adapter~\cite{houlsby2019parameter} and low-rank~\cite{hu2021lora} tuning freeze a
backbone and train a small injected parameter set; LoRA is now a default for medical
vision-language adaptation, but such methods control \emph{where} a model changes
without bounding earlier-competence loss or representation drift. Sequential training
induces catastrophic forgetting~\cite{mccloskey1989catastrophic}, addressed by
regularization, replay and architectural
families~\cite{parisi2019continual,delange2022continual}; replay and exemplar
methods~\cite{lopez2017gradient,rebuffi2017icarl} retain past data, often clinically
unacceptable for privacy, while regularization methods weight movement by an importance
estimate~\cite{kirkpatrick2017overcoming,zenke2017synaptic,aljundi2018memory}, online
EWC~\cite{schwarz2018progress} keeping a single running term and functional methods
regularizing in output space~\cite{li2017learning,titsias2020functional}. We build on
online EWC, remove its scale-and-order fragility (Sec.~\ref{sec:method}), and anchor in
\emph{embedding} space to the \emph{frozen prior} rather than distilling from the
previous model. Composition methods allocate capacity per task via prompt
pools~\cite{wang2022dualprompt,jia2022vpt} or orthogonalized
adapters~\cite{wang2023olora,liang2024inflora}, growing parameters with the task count
or adding inference-time selection; CADRE keeps a \emph{single} shared adapter, though
we do not compare against that family empirically (Sec.~\ref{sec:discussion}).

\textbf{Medical continual learning, calibration and weight averaging.}
EWC and learning without forgetting (LwF) have been studied for chest-radiograph
classification~\cite{lenga2020continual}, and dynamic memory for acquisition
shift~\cite{perkonigg2021dynamic,medclsurvey2024}. Closest to us, recent work adapts the
EWC penalty for medical VLMs via prompt-guided grouping and similarity-weighted
Fisher~\cite{paewc2025}; CADRE differs in requiring \emph{no} hand-tuned strength, in
making the penalty scale-free via a sum-normalised Fisher with a proven mass bound, and
in adding an explicit anchor-to-prior term. We adapt
BiomedCLIP~\cite{zhang2023biomedclip} rather than train a VLM~\cite{wang2022medclip},
and use Expected Calibration Error (ECE)~\cite{guo2017calibration}, label
smoothing~\cite{szegedy2016rethinking} and weight
averaging~\cite{izmailov2018averaging,tarvainen2017mean}.

\section{Method: CADRE}
\label{sec:method}
\textbf{Adaptation model.} The BiomedCLIP backbone (ViT-B/16 vision encoder, PubMedBERT
text encoder) stays frozen; we train only a parameter vector $\theta$ ($\approx$0.23\%
of the backbone) of LoRA factors injected into the attention projections of the 25
visual-encoder layers, plus the linear head. We write $g_{\text{LoRA}}$ and
$g_{\text{frozen}}$ for the adapted and frozen-prior image embeddings. Modalities arrive
as a stream $\mathcal{D}_1,\dots,\mathcal{D}_T$, each a balanced binary task.

\textbf{Metrics.} Let $A_{t,j}$ be accuracy on modality $j$ after training through
modality $t$. We report average final accuracy and the area under the
receiver-operating-characteristic curve (AUROC); \emph{forgetting}, the mean drop from
each earlier modality's \emph{best} accuracy to its \emph{final} accuracy;
\emph{backward transfer} (BWT), the mean change in an earlier modality's accuracy
between the point at which it was \emph{learned} and the final state; and the
\emph{Stability-Plasticity Quotient} (SPQ), the harmonic mean of plasticity (mean
accuracy on each modality when first learned) and stability (one minus forgetting,
floored at zero). Lower forgetting and higher BWT and SPQ are better; the first two use
different reference points, best-to-final and learned-to-final.

\textbf{Stability objectives.} An update should satisfy three monitorable properties:
\textbf{(S1)} bounded competence loss, \textbf{(S2)} bounded prior drift and
\textbf{(S3)} usable calibration, targeted respectively by the consolidation term, the
anchor and the calibration components. (S1) and (S3) are measured below; (S2) is imposed
by construction and not measured (Sec.~\ref{sec:discussion}).

% ---------------------------------------------------------------------------
% Fig. 1 -- CADRE architecture. Pure TikZ, greyscale only, no external image.
% Explicit millimetre grid (x=1mm, y=1mm) spanning 0--122mm, the LNCS text width.
%   * every node has a fixed text width, so no box can grow into its neighbour
%   * horizontal gaps are a uniform 5mm; band 1 boxes share one 12mm height
%   * bands are placed by their north anchor, so box tops are flush per band
%   * connectors use node anchors (-|, |-), never hand-typed endpoints
% Vertical grid: band1 y=0 | adapter/prior -14 | objective -33 | mechanisms -49
% ---------------------------------------------------------------------------
\begin{figure}[t]
\centering
\resizebox{\textwidth}{!}{%
\begin{tikzpicture}[x=1mm, y=1mm, font=\scriptsize,
  >={Stealth[length=1.6mm]},
  every node/.style = {inner xsep=1.5mm, inner ysep=1mm},
  box/.style    = {draw=black!65, rounded corners=0.6mm, align=center,
                   fill=white, minimum height=12mm},
  frozen/.style = {box, draw=black, fill=black!10},
  train/.style  = {box, draw=black, line width=0.9pt},
  item/.style   = {draw=black!65, rounded corners=0.6mm, align=center,
                   fill=white, minimum height=6.5mm, text width=21mm},
  mech/.style   = {draw=black!65, rounded corners=0.6mm, align=left, anchor=north,
                   fill=black!4, text width=34.5mm, minimum height=16mm},
  fwd/.style    = {->, draw=black, line width=0.45pt},
  pri/.style    = {->, draw=black!75, line width=0.45pt, densely dotted},
  grad/.style   = {->, draw=black, line width=0.7pt, dashed},
  lb/.style     = {inner xsep=0.6mm, inner ysep=0.4mm}]

% ===== band 1 : forward path, all boxes centred on y = 0 ===================
\node[item] (d1) at (13.5,  9) {$\mathcal{D}_1$ histopathology};
\node[item] (d2) at (13.5,  0) {$\mathcal{D}_2$ ultrasound};
\node[item] (d3) at (13.5, -9) {$\mathcal{D}_3$ radiography};
\draw[fwd] (d1) -- (d2);
\draw[fwd] (d2) -- (d3);
\node[draw=black!45, densely dashed, rounded corners=1mm, inner sep=1.5mm,
      fit=(d1)(d2)(d3)] (stream) {};
\node[lb, anchor=south west] at (stream.north west) {sequential modality stream};

\node[frozen, text width=26mm] (enc)  at ( 46.5, 0)
      {BiomedCLIP\\ViT-B/16 encoder\\\emph{frozen}};
\node[train,  text width=15mm] (head) at ( 75.0, 0) {linear head\\\emph{trainable}};
\node[box,    text width=13mm] (pred) at ( 97.0, 0) {softmax $\hat{y}$};
\node[box,    text width= 9mm] (lab)  at (116.0, 0) {label $y$};

\draw[fwd] (stream.east) -- (enc.west);
\draw[fwd] (enc.east)    -- (head.west);
\draw[fwd] (head.east)   -- (pred.west);
\draw[fwd] (pred.east)   -- (lab.west);
\node[lb] at (63.5, 8.6) {$g_{\text{LoRA}}$};

% ===== band 2 : adapter and frozen prior, tops flush at y = -14 ============
\node[train, anchor=north, text width=41mm, minimum height=14mm] (lora)
      at ( 54.0,-14)
      {visual LoRA \emph{trainable} --- rank~8, $\alpha{=}16$,
       $\approx$0.23\% of parameters,\\injected into the attention projections};
\node[frozen, anchor=north, text width=38mm, minimum height=14mm, densely dashed]
      (prior) at (101.5,-14)
      {$g_{\text{frozen}}$: the same encoder\\with its LoRA paths off\\(cached prior)};

\coordinate (tap) at (58,0);
\draw[<->, draw=black, line width=0.45pt] (enc.south) -- (enc.south |- lora.north);
\draw[pri] (enc.south -| tap) -- ++(0,-4) -| (prior.north);

% ===== legend ==============================================================
\draw[fwd]  ( 2,-33) -- ( 8,-33);  \node[lb, anchor=west] at ( 9,-33) {forward path};
\draw[pri]  (28,-33) -- (34,-33);  \node[lb, anchor=west] at (35,-33) {prior path};
% \draw[grad] 
\node[draw=black, fill=black!10, rounded corners=0.4mm, inner sep=0pt,
      minimum width=6mm, minimum height=3mm] at (60,-33) {};
\node[lb, anchor=west] at (64,-33) {frozen};
\node[draw=black, fill=white, line width=0.9pt, rounded corners=0.4mm, inner sep=0pt,
      minimum width=6mm, minimum height=3mm] at (80,-33) {};
\node[lb, anchor=west] at (84,-33) {trainable};

\end{tikzpicture}}
\caption{Overview of CADRE. A frozen BiomedCLIP encoder is adapted sequentially over
dissimilar modalities through one shared LoRA adapter and a linear head
($\approx$0.23\% of parameters). The prior $g_{\text{frozen}}$ is the same encoder with
its LoRA paths off, so the anchor compares two states of one model, not two models.}
\label{fig:method}
\end{figure}
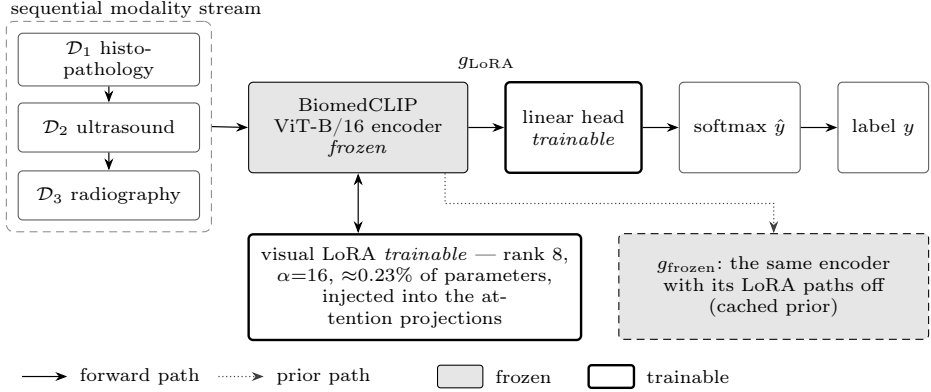

\textbf{Why vanilla EWC is order-fragile.} EWC weights each parameter by its Fisher
importance and pulls it back toward a stored reference $\theta^{*}$; online
EWC~\cite{schwarz2018progress} keeps a single running estimate $\mathcal{F}$ and one
reference rather than one per task:
\begin{equation}
\mathcal{L}_{\text{EWC}}(\theta)=\textstyle\sum_i \mathcal{F}_i\,(\theta_i-\theta^{*}_i)^2,
\quad
\mathcal{F}\leftarrow \gamma_{\text{sim}}\,\mathcal{F}+\hat{\mathcal{F}}^{(t)},
\quad \theta^{*}\leftarrow\theta_t .
\label{eq:ewc}
\end{equation}
A single fixed multiplier $\lambda$ on the raw Fisher creates two coupled scale
problems. The raw Fisher tracks gradient magnitude, so it varies with task difficulty
and loss scale, and the running sum in \eqref{eq:ewc} grows as tasks are consolidated;
one $\lambda$ cannot stay comparable across steps. That sum is also order-dependent, so
a $\lambda$ tuned for one order over- or under-constrains another and the last-arriving
modality faces the largest accumulated penalty.

\textbf{Scale-decoupling.} \textbf{(M1) Sum-normalised Fisher.} Before consolidating
each task's Fisher we rescale it to unit total mass, so every modality contributes the
same total importance regardless of its loss scale and the running estimate in
\eqref{eq:ewc} can never accumulate unbounded mass. \textbf{(M2) Self-scaling
strength.} Rather than tune a fixed $\lambda$, we recompute it at each step from
detached (stop-gradient) loss values,
\begin{equation}
\lambda_t=\min\!\Big(\rho\,\tfrac{\mathcal{L}_{\text{CE}}}{\mathcal{L}_{\text{EWC}}+\varepsilon},\ \lambda_{\max}\Big),\qquad \rho=0.3 ,
\label{eq:selfscale}
\end{equation}
so that, when unclipped, the penalty always contributes a fixed fraction $\rho$ of the
current task loss; a floor $\varepsilon$, a cap $\lambda_{\max}$ and a 50-step warm-up
(with $\lambda_t{=}0$ while the first Fisher is estimated) stop it spiking before the
prior is meaningfully perturbed. \textbf{(M3) Similarity-aware retention.} Keeping old
constraints in full over-regularises an incompatible target, so we scale retention of
the running Fisher by the cosine between successive modality prototypes $p_t$ (means of
normalised embeddings over the modality's training split):
\begin{equation}
\gamma_{\text{sim}}=\gamma\cdot\tfrac{1}{2}\big(1+\cos(p_t,p_{t-1})\big),
\qquad \gamma=0.9 ,
\label{eq:simretain}
\end{equation}
so dissimilar modalities relax retention and similar ones preserve it, and since
$\cos(\cdot,\cdot)\in[-1,1]$ this keeps $\gamma_{\text{sim}}\in[0,\gamma]$, the
precondition of Proposition~\ref{prop:mass}. At this scale M3 has no measurable effect
on forgetting, so we present it as an optional refinement rather than a supported
contribution.

\textbf{Drift bound: anchor-to-prior.} To control (S2) we add an embedding-space
penalty~\cite{li2017learning,titsias2020functional} keeping the adapted encoder close to
the frozen prior on a fixed probe set $\mathcal{A}$ ($|\mathcal{A}|{=}64$):
\begin{equation}
\mathcal{L}_{\text{anchor}}=\frac{1}{|\mathcal{A}|}\sum_{x\in\mathcal{A}}
\big\|g_{\text{LoRA}}(x)-g_{\text{frozen}}(x)\big\|_2^2 .
\label{eq:anchor}
\end{equation}
$\mathcal{A}$ is fixed at the start of the first modality, drawn from the \emph{training}
split only with its prior embeddings cached, so it cannot leak into evaluation, and the
two embeddings come from \emph{different} adapter states on the \emph{same} input.
Adding $\beta\,\mathcal{L}_{\text{anchor}}$ is the Lagrangian of minimising the task loss
subject to a cap on \emph{mean} drift over $\mathcal{A}$: a soft bound on
\emph{expected} drift, not a per-input Lipschitz guarantee.

\textbf{Calibration and the full objective.} For (S3), CADRE adds light label smoothing
($\varepsilon_{\text{ls}}{=}0.05$) and an exponential moving average (EMA, decay $0.99$)
of the trainable parameters applied at \emph{evaluation} only, in the manner of
stochastic weight averaging~\cite{izmailov2018averaging}; ECE is the mean gap between
confidence and accuracy across equal-width bins. The per-step loss is the label-smoothed
cross-entropy plus the self-scaled consolidation penalty (\eqref{eq:selfscale},
sum-normalised and similarity-decayed) plus the anchor penalty. During training the
Fisher and reference $\theta^{*}$ come from the live weights, so the penalty and the
optimised parameters are always the same set; only evaluation uses EMA weights, and
because EMA can itself reduce measured forgetting we \emph{attribute} its contribution
rather than credit it to consolidation.

\textbf{Theoretical analysis.} Two short results explain why the scale fixes work.
\begin{proposition}[Bounded consolidation mass]\label{prop:mass}
If each consolidated Fisher has unit mass and the retention factor satisfies
$\gamma_{\text{sim}}\in[0,\gamma]$ with $0\le\gamma<1$, then the total mass
$m_t=\sum_i\mathcal{F}_i$ stays bounded by $1/(1-\gamma)$ for all $t$, independent of
the number of modalities $T$.
\end{proposition}
\noindent\emph{Proof.} The update in \eqref{eq:ewc} gives
$m_t=\gamma_{\text{sim}}m_{t-1}+1\le\gamma m_{t-1}+1$; with $m_0=0$, induction yields
$m_t\le\frac{1-\gamma^t}{1-\gamma}\le\frac{1}{1-\gamma}$.\hfill$\square$

\begin{proposition}[Scale-invariant trade-off]\label{prop:scale}
Let $\varepsilon\to0$ with $\lambda_t$ unclipped. Under any rescaling of the penalty
$\mathcal{L}_{\text{EWC}}\mapsto c\,\mathcal{L}_{\text{EWC}}$ ($c>0$), both the
self-scaled penalty value and its gradient are unchanged.
\end{proposition}
\noindent\emph{Proof.} By \eqref{eq:selfscale} the multiplier rescales as
$\lambda_t\mapsto\lambda_t/c$, which exactly cancels the factor $c$ in both the
penalty value ($\rho\,\mathcal{L}_{\text{CE}}$) and its gradient.\hfill$\square$

M1 and M2 remove the two \emph{scale-related} sources of order dependence but not
interference between modality representations, which is not a scale effect, so we claim
order-robustness only \emph{empirically}. After each modality we update the prototype,
consolidate the Fisher and reference, and re-evaluate all modalities seen so far under
the EMA weights.

\section{Experimental Setup}
\label{sec:setup}
\textbf{Data and task framing.} We use breast histopathology, breast ultrasound and
chest radiography as a deliberately \emph{dissimilar} cross-modality stress test. Each
modality is balanced ($N{=}600$, 50\% positive; $N{=}1800$ total), read as RGB at
$224\times224$; the histopathology source is BreaKHis~\cite{spanhol2016breakhis}, whose
multispectral folder labels are a packaging artefact rather than extra spectral bands.
Chest radiography is an \emph{adversarial negative-transfer probe}, not a clinical task:
it is most distant from breast tissue, its abnormal vs.\ normal labels are not
biopsy-correlated, and no breast-malignancy claim attaches to it. The three tasks use
related but non-identical predicates, so average accuracy is a retention diagnostic for
a shared adapter rather than a clinical performance figure; the claims we defend
(forgetting, BWT and their variance across orders) are computed \emph{within} a modality
and are unaffected by that mismatch. Histopathology is grouped at slide level (43
groups) via BreaKHis SOB filename parsing; ultrasound and radiography group at image
level, since no case identifiers are recoverable from the distributed files, so
same-case leakage cannot be excluded and those two absolute accuracies are likely
optimistic. Splits are stratified by group into train/validation/test ($70/15/15$).

\textbf{Protocol.} All methods adapt a frozen BiomedCLIP backbone under an identical
loop, differing only in the regularization applied: visual LoRA of rank $8$ with
$\alpha{=}16$ on the attention projections, AdamW at learning rate $5\times10^{-4}$ and
weight decay $10^{-4}$, ten epochs per modality, mixed precision, one T4 GPU; CADRE adds
$\rho{=}0.3$, $\gamma{=}0.9$, $\beta{=}0.3$, $|\mathcal{A}|{=}64$, label smoothing
$0.05$, EMA decay $0.99$, and a Fisher over 50 batches. Each run adapts over the three
modalities in turn, re-evaluating all modalities seen so far after each, giving the
matrix $A_{t,j}$. We ran 3 seeds $\times$ 2 arrival orders, chosen so histopathology,
whose retention is most fragile under the naive baseline, appears both first and last,
for $n{=}6$ paired observations per method; three modalities admit $3!{=}6$ orders, so
this samples the order space rather than covering it. Comparisons are \emph{paired}
across the six (seed, order) pairs, two-sided at $5$ degrees of freedom, as mean $\pm$
standard error of the mean (SEM); we apply no multiple-comparison correction and treat
$p$-values as descriptive. The methods compared are \emph{Linear Probe} (frozen
features, linear head), \emph{LoRA}, \emph{LoRA$+$EWC} (vanilla EWC, fixed $\lambda$,
unnormalized Fisher), \emph{LoRA$+$Anchor} and \emph{CADRE}; the LoRA$+$EWC contrast
reflects both the redesign and the added components, separated in
Table~\ref{tab:ablation}. Code, configurations and seed and order lists are released at
\url{https://github.com/rishabhjha1/CADRE}; the retention mapping is
\eqref{eq:simretain}.

\section{Results}
\label{sec:results}
All values are mean$\pm$SEM over 3 seeds $\times$ 2 orders from one unified run, so
comparison, attribution and accuracy matrices are mutually consistent. Among the
adapting methods evaluated, CADRE attains the highest accuracy ($0.775$), SPQ ($0.867$)
and backward transfer ($+0.004$) and the lowest forgetting ($0.011$;
Table~\ref{tab:main}); AUROC is flat across the LoRA family, from $0.813$ to $0.821$,
within roughly one SEM.

\textbf{Forgetting is the headline result.} CADRE reduces forgetting roughly sevenfold
against the strongest regularized baseline, LoRA$+$Anchor ($0.075\!\rightarrow\!0.011$;
paired $p{=}0.023$), and against plain LoRA ($0.084\!\rightarrow\!0.011$; $p{=}0.002$),
and is the only method in Table~\ref{tab:main} with \emph{positive} backward transfer
($+0.004$) where every baseline is clearly negative ($-0.07$ to $-0.08$). The comparison
with LoRA$+$EWC does not reach significance ($p{=}0.072$); that baseline is
high-variance (forgetting SEM $\pm0.026$ against CADRE's $\pm0.006$), so at $n{=}6$ we
can neither claim nor dismiss the difference and report it as unresolved. On accuracy,
paired $t$-tests place CADRE above plain LoRA ($+0.030$, $p{=}0.014$), LoRA$+$Anchor
($+0.021$, $p{=}0.008$) and the Linear Probe ($+0.173$, $p{=}0.002$), the LoRA$+$EWC gap
again not significant ($+0.020$, $p{=}0.372$).

\begin{table}[t]
\centering
\caption{Main comparison (mean$\pm$SEM, 3 seeds $\times$ 2 orders, $n{=}6$). Lower
forgetting is better; higher Acc, AUROC, BWT, SPQ are better. Best per column in
\textbf{bold}; AUROC is not bolded, as its spread across LoRA variants is within noise.
All LoRA variants train $0.231\%$ of parameters (linear probe $0.001\%$) with one shared
adapter, no exemplars and no inference-time routing.}
\label{tab:main}
\small
\setlength{\tabcolsep}{4pt}
\begin{tabular}{lccccc}
\toprule
Method & Acc & AUROC & BWT & Forget. & SPQ \\
\midrule
Linear Probe & $0.603${\scriptsize$\pm.029$} & $0.720${\scriptsize$\pm.008$} & $-0.009${\scriptsize$\pm.046$} & $0.061${\scriptsize$\pm.022$} & $0.738${\scriptsize$\pm.006$} \\
LoRA & $0.746${\scriptsize$\pm.012$} & $0.813${\scriptsize$\pm.011$} & $-0.081${\scriptsize$\pm.017$} & $0.084${\scriptsize$\pm.015$} & $0.854${\scriptsize$\pm.008$} \\
LoRA $+$ EWC & $0.756${\scriptsize$\pm.021$} & $0.820${\scriptsize$\pm.009$} & $-0.070${\scriptsize$\pm.028$} & $0.081${\scriptsize$\pm.026$} & $0.856${\scriptsize$\pm.013$} \\
LoRA $+$ Anchor & $0.754${\scriptsize$\pm.011$} & $0.813${\scriptsize$\pm.009$} & $-0.075${\scriptsize$\pm.016$} & $0.075${\scriptsize$\pm.016$} & $0.860${\scriptsize$\pm.007$} \\
CADRE (full) & $\mathbf{0.775}${\scriptsize$\pm.009$} & $0.821${\scriptsize$\pm.006$} & $\mathbf{0.004}${\scriptsize$\pm.008$} & $\mathbf{0.011}${\scriptsize$\pm.006$} & $\mathbf{0.867}${\scriptsize$\pm.006$} \\
\bottomrule
\end{tabular}
\end{table}

\textbf{Per-modality behaviour and order robustness.} CADRE collapses neither an early
modality nor a class: it achieves the best abnormal-class recall on the hardest
modality, chest radiography ($0.607$ against $0.389$ to $0.511$ for the LoRA baselines),
while remaining strong on histopathology ($0.856$ / $0.776$ by class) and ultrasound
($0.911$ / $0.930$). Order sensitivity surfaces as variance across (seed, order) pairs:
CADRE's forgetting SEM is $\pm0.006$ on a mean of $0.011$, four times tighter than
LoRA$+$EWC's $\pm0.026$ on $0.081$, and its BWT SEM is $\pm0.008$ against $\pm0.028$,
consistent with the scale guarantees of Sec.~\ref{sec:method}, though two of six orders
is evidence for, not a demonstration of, order-invariance. The $A_{t,j}$ matrix on the
first order shows histopathology near-flat ($0.857\!\to\!0.870\!\to\!0.844$) and
chest-radiograph accuracy \emph{improving} after a later modality
($0.581\!\to\!0.589$).

\textbf{Component attribution.} Table~\ref{tab:ablation} toggles each component off with
the rest held fixed; CADRE-full's forgetting in this arm ($0.009\pm0.004$) matches the
main comparison ($0.011\pm0.006$), so attribution and headline share one protocol. Three
removals significantly raise forgetting: the consolidation redesign, the anchor and EMA.
Reverting the redesign costs the most accuracy ($0.777\!\rightarrow\!0.749$) and is the
primary stability driver; the anchor's clearest effect is on accuracy ($p{=}0.005$). EMA
slightly \emph{lowers} accuracy when present ($0.784\!\rightarrow\!0.777$) while
suppressing forgetting, as expected, since trajectory averaging damps late-stream
movement mechanically. Label smoothing and M3 show no significant effect at $n{=}6$; for
M3 mean forgetting is the wrong endpoint, since it targets \emph{variance across
orders}, which two orders cannot test.

\begin{table}[t]
\centering
\caption{Component attribution: each row removes one component, the rest held fixed (3
seeds $\times$ 2 orders; $p$ on forgetting against CADRE-full, paired, $5$ degrees of
freedom). Rows are ablations of CADRE, not competing methods; removing the anchor also
lowers accuracy significantly ($p{=}0.005$).}
\label{tab:ablation}
\small
\setlength{\tabcolsep}{6pt}
\begin{tabular}{lcccc}
\toprule
Configuration & Forgetting & $p$ vs.\ full & Acc & ECE \\
\midrule
CADRE (full) & $0.009${\scriptsize$\pm.004$} & ---& $0.777$ & $0.140$ \\
\midrule
$-$ self-scaling (revert to vanilla EWC) & $0.038$ & $0.035$ & $0.749$ & $0.129$ \\
$-$ EMA & $0.025$ & $0.042$ & $0.784$ & $0.155$ \\
$-$ anchor & $0.019$ & $0.019$ & $0.744$ & $0.134$ \\
$-$ label smoothing & $0.018$ & $0.36$ & $0.778$ & $0.147$ \\
$-$ similarity-aware retention (M3) & $0.016$ & $0.49$ & $0.770$ & $0.139$ \\
\bottomrule
\end{tabular}
\end{table}

\section{Discussion and Conclusion}
\label{sec:discussion}
Treating forgetting as a quantity to bound turns a continual-learning metric into a
stability property a clinical team can monitor: a precondition for safe updating, not a
proof of it. The substantive finding is narrow. Online EWC has a diagnosable,
scale-induced sensitivity to arrival order, and normalising the Fisher while tying the
multiplier to the task loss removes it without a replacement hyper-parameter;
empirically the redesign is the largest contributor to both retention and accuracy.

\textbf{Limitations.} All comparison methods share CADRE's single-shared-adapter
setting; we do not compare empirically against capacity-expanding, prompt-based or
distillation-based continual
learners~\cite{wang2023olora,liang2024inflora,wang2022dualprompt,li2017learning,paewc2025},
so our claims concern the regularization family at a fixed parameter budget. At $n{=}6$
over two of six orders, the LoRA$+$EWC forgetting comparison is unresolved rather than
null ($p{=}0.072$); a full $5\times6$ grid with a $\lambda$ sweep
and a plot of consolidation mass against $t$ is the direct empirical test of
Propositions~\ref{prop:mass} and~\ref{prop:scale} and is not carried out here. Objective
(S2) is correspondingly weaker than (S1) and (S3): the anchor imposes a soft bound on
expected drift by construction, but we report no measured drift statistic and no runtime
out-of-distribution guardrail result, so bounded drift is a design property rather than
an empirical finding. Image-level grouping leaves same-case leakage possible, and the
benchmark is balanced and in-distribution, so neither class imbalance nor acquisition
shift is probed.

\textbf{Conclusion.} CADRE recasts parameter-efficient continual adaptation as a
stability problem. A self-scaling EWC redesign with a bounded-mass guarantee and a
scale-invariance property, plus an anchor-to-prior drift penalty, gives roughly
sevenfold lower forgetting than the strongest LoRA-family baseline, positive backward
transfer and lower cross-order variance at $\approx$0.23\% of parameters: monitorable
quantities aligned with clinical-safety desiderata, not a deployment guarantee.
Code: \url{https://github.com/rishabhjha1/CADRE}.

\begin{credits}
\subsubsection{\discintname}
The authors have no competing interests to declare.
\end{credits}


\begin{thebibliography}{99}

\bibitem{zhang2023biomedclip}
Zhang, S., Xu, Y., Usuyama, N., Bagga, J., Tinn, R., Preston, S., Rao, R., Wei, M.,
Valluri, N., Wong, C., Lungren, M.P., Naumann, T., Poon, H.:
BiomedCLIP: a multimodal biomedical foundation model pretrained from fifteen million
image--text pairs.
arXiv:2303.00915 (2023)

\bibitem{mccloskey1989catastrophic}
McCloskey, M., Cohen, N.J.:
Catastrophic interference in connectionist networks: the sequential learning problem.
In: Bower, G.H. (ed.) Psychology of Learning and Motivation, vol.~24, pp.~109--165.
Academic Press (1989)

\bibitem{kirkpatrick2017overcoming}
Kirkpatrick, J., Pascanu, R., Rabinowitz, N., Veness, J., Desjardins, G., Rusu, A.A.,
Milan, K., Quan, J., Ramalho, T., Grabska-Barwinska, A., Hassabis, D., Clopath, C.,
Kumaran, D., Hadsell, R.:
Overcoming catastrophic forgetting in neural networks.
Proceedings of the National Academy of Sciences \textbf{114}(13), 3521--3526 (2017)

\bibitem{schwarz2018progress}
Schwarz, J., Czarnecki, W., Luketina, J., Grabska-Barwinska, A., Teh, Y.W.,
Pascanu, R., Hadsell, R.:
Progress \& Compress: a scalable framework for continual learning.
In: Proceedings of the 35th International Conference on Machine Learning (ICML),
pp.~4528--4537 (2018)

\bibitem{hu2021lora}
Hu, E.J., Shen, Y., Wallis, P., Allen-Zhu, Z., Li, Y., Wang, S., Wang, L., Chen, W.:
LoRA: low-rank adaptation of large language models.
In: International Conference on Learning Representations (ICLR) (2022).
arXiv:2106.09685

\bibitem{li2017learning}
Li, Z., Hoiem, D.:
Learning without forgetting.
IEEE Transactions on Pattern Analysis and Machine Intelligence \textbf{40}(12),
2935--2947 (2017)

\bibitem{lopez2017gradient}
Lopez-Paz, D., Ranzato, M.:
Gradient episodic memory for continual learning.
In: Advances in Neural Information Processing Systems 30 (NeurIPS),
pp.~6467--6476 (2017)

\bibitem{rebuffi2017icarl}
Rebuffi, S.A., Kolesnikov, A., Sperl, G., Lampert, C.H.:
iCaRL: incremental classifier and representation learning.
In: Proceedings of the IEEE Conference on Computer Vision and Pattern Recognition
(CVPR), pp.~2001--2010 (2017)

\bibitem{wang2022medclip}
Wang, Z., Wu, Z., Agarwal, D., Sun, J.:
MedCLIP: contrastive learning from unpaired medical images and text.
In: Proceedings of the 2022 Conference on Empirical Methods in Natural Language
Processing (EMNLP), pp.~3894--3905 (2022). arXiv:2210.10163

\bibitem{parisi2019continual}
Parisi, G.I., Kemker, R., Part, J.L., Kanan, C., Wermter, S.:
Continual lifelong learning with neural networks: a review.
Neural Networks \textbf{113}, 54--71 (2019)

\bibitem{delange2022continual}
De Lange, M., Aljundi, R., Masana, M., Parisot, S., Jia, X., Leonardis, A.,
Slabaugh, G., Tuytelaars, T.:
A continual learning survey: defying forgetting in classification tasks.
IEEE Transactions on Pattern Analysis and Machine Intelligence \textbf{44}(7),
3366--3385 (2022)

\bibitem{zenke2017synaptic}
Zenke, F., Poole, B., Ganguli, S.:
Continual learning through synaptic intelligence.
In: Proceedings of the 34th International Conference on Machine Learning (ICML),
pp.~3987--3995 (2017)

\bibitem{aljundi2018memory}
Aljundi, R., Babiloni, F., Elhoseiny, M., Rohrbach, M., Tuytelaars, T.:
Memory aware synapses: learning what (not) to forget.
In: Proceedings of the European Conference on Computer Vision (ECCV),
pp.~144--160 (2018)

\bibitem{titsias2020functional}
Titsias, M.K., Schwarz, J., Matthews, A.G.G., Pascanu, R., Teh, Y.W.:
Functional regularisation for continual learning with Gaussian processes.
In: International Conference on Learning Representations (ICLR) (2020)

\bibitem{houlsby2019parameter}
Houlsby, N., Giurgiu, A., Jastrzebski, S., Morrone, B., de Laroussilhe, Q.,
Gesmundo, A., Attariyan, M., Gelly, S.:
Parameter-efficient transfer learning for NLP.
In: Proceedings of the 36th International Conference on Machine Learning (ICML),
pp.~2790--2799 (2019)

\bibitem{guo2017calibration}
Guo, C., Pleiss, G., Sun, Y., Weinberger, K.Q.:
On calibration of modern neural networks.
In: Proceedings of the 34th International Conference on Machine Learning (ICML),
pp.~1321--1330 (2017)

\bibitem{szegedy2016rethinking}
Szegedy, C., Vanhoucke, V., Ioffe, S., Shlens, J., Wojna, Z.:
Rethinking the Inception architecture for computer vision.
In: Proceedings of the IEEE Conference on Computer Vision and Pattern Recognition
(CVPR), pp.~2818--2826 (2016)

\bibitem{izmailov2018averaging}
Izmailov, P., Podoprikhin, D., Garipov, T., Vetrov, D., Wilson, A.G.:
Averaging weights leads to wider optima and better generalization.
In: Proceedings of the Thirty-Fourth Conference on Uncertainty in Artificial
Intelligence (UAI), pp.~876--885 (2018)

\bibitem{tarvainen2017mean}
Tarvainen, A., Valpola, H.:
Mean teachers are better role models: weight-averaged consistency targets improve
semi-supervised deep learning results.
In: Advances in Neural Information Processing Systems 30 (NeurIPS),
pp.~1195--1204 (2017)

\bibitem{lenga2020continual}
Lenga, M., Schulz, H., Saalbach, A.:
Continual learning for domain adaptation in chest X-ray classification.
In: Proceedings of the Third Conference on Medical Imaging with Deep Learning
(MIDL), PMLR 121, pp.~413--423 (2020). arXiv:2001.05922

\bibitem{perkonigg2021dynamic}
Perkonigg, M., Hofmanninger, J., Herold, C.J., Brink, J.A., Pianykh, O.S.,
Prosch, H., Langs, G.:
Dynamic memory to alleviate catastrophic forgetting in continual learning with
medical imaging.
Nature Communications \textbf{12}, 5678 (2021)

\bibitem{medclsurvey2024}
Qazi, M.A., Nawaz, M., Naseer, M., Khan, S., Khan, F.S.:
Continual learning in medical imaging: a survey and practical analysis.
arXiv:2405.13482 (2024)

\bibitem{paewc2025}
Gao, Z., Morel, P.:
Prompt-aware adaptive elastic weight consolidation for continual learning in
medical vision-language models.
arXiv:2511.20732 (2025)

\bibitem{spanhol2016breakhis}
Spanhol, F.A., Oliveira, L.S., Petitjean, C., Heutte, L.:
A dataset for breast cancer histopathological image classification.
IEEE Transactions on Biomedical Engineering \textbf{63}(7), 1455--1462 (2016)

\bibitem{jia2022vpt}
Jia, M., Tang, L., Chen, B.C., Cardie, C., Belongie, S., Hariharan, B., Lim, S.N.:
Visual prompt tuning.
In: Proceedings of the European Conference on Computer Vision (ECCV),
Lecture Notes in Computer Science 13686, pp.~709--727 (2022)

\bibitem{wang2022dualprompt}
Wang, Z., Zhang, Z., Ebrahimi, S., Sun, R., Zhang, H., Lee, C.Y., Ren, X., Su, G.,
Perot, V., Dy, J., Pfister, T.:
DualPrompt: complementary prompting for rehearsal-free continual learning.
In: Proceedings of the European Conference on Computer Vision (ECCV),
Lecture Notes in Computer Science 13686, pp.~13746--13756 (2022)

\bibitem{wang2023olora}
Wang, X., Chen, T., Ge, Y., Zhou, J., Zhao, H.:
Orthogonal subspace learning for language model continual learning (O-LoRA).
In: Findings of the Association for Computational Linguistics: EMNLP 2023,
pp.~11123--11137 (2023). arXiv:2310.14152

\bibitem{liang2024inflora}
Liang, Y.S., Li, W.J.:
InfLoRA: interference-free low-rank adaptation for continual learning.
In: Proceedings of the IEEE/CVF Conference on Computer Vision and Pattern
Recognition (CVPR), pp.~28672--28681 (2024)

\end{thebibliography}
\end{document}